\documentclass{article}

\usepackage[utf8]{inputenc} 
\usepackage[T1]{fontenc}    
\usepackage{hyperref}       
\usepackage{url}            
\usepackage{booktabs}       
\usepackage{multirow}
\usepackage{amsfonts}       
\usepackage{nicefrac}       
\usepackage{microtype}      
\usepackage{xcolor}         
\usepackage{wrapfig}
\usepackage{float} 
\usepackage{graphicx} 
\usepackage{subfigure}

\usepackage{bm} 

\usepackage[margin=1.03in]{geometry} 

\usepackage{algorithm} 
\usepackage{algorithmic}

\usepackage{enumerate} 
\usepackage{natbib} 

\usepackage{amsmath}
\usepackage{amssymb}
\usepackage{mathtools}
\usepackage{amsthm}
\usepackage{mathrsfs}

\renewcommand{\Pr}{ \mathbb{P} }

\newcommand{\A}{ \mathcal{A} }

\newcommand{\R}{\mathbb{R}}

\newcommand{\E}{\mathbb{E}}

\newcommand{\m}{\mathrm{m}}

\theoremstyle{plain}
\newtheorem{theorem}{Theorem}

\newtheorem{lemma}{Lemma}

\theoremstyle{definition}
\newtheorem{definition}{Definition}
\newtheorem{assumption}{Assumption}
\newtheorem{remark}{Remark}

\title{A Lipschitz Bandits Approach for Continuous Hyperparameter Optimization}

\author{Yasong Feng\qquad
	Weijian Luo\qquad
    Yimin Huang\qquad
	Tianyu Wang
}
\date{}

\begin{document}

\maketitle

\begin{abstract}
One of the most critical problems in machine learning is HyperParameter Optimization (HPO), since choice of hyperparameters has a significant impact on final model performance. Although there are many HPO algorithms, they either have no theoretical guarantees or require strong assumptions. To this end, we introduce BLiE -- a Lipschitz-bandit-based algorithm for HPO that only assumes Lipschitz continuity of the objective function. BLiE exploits the landscape of the objective function to adaptively search over the hyperparameter space. Theoretically, we show that $(i)$ BLiE finds an $\epsilon$-optimal hyperparameter with $\mathcal{O} \left( \epsilon^{-(d_z + \beta)}\right)$ total budgets, where $d_z$ and $\beta$ are problem intrinsic; $(ii)$ BLiE is highly parallelizable. Empirically, we demonstrate that BLiE outperforms the state-of-the-art HPO algorithms on benchmark tasks. We also apply BLiE to search for noise schedule of diffusion models. Comparison with the default schedule shows that BLiE schedule greatly improves the sampling speed.
\end{abstract}

\setlength{\parindent}{0pt}

\section{Introduction}
Success of modern machine learning models heavily relies on the choice of hyperparameters. These hyperparameters are difficult to set, because of high training cost of the complex models. 
Therefore, practitioners are in great need of efficient algorithms for finding good hyperparameter configurations.


\begin{wrapfigure}{r}{0.5\textwidth}
    \centering
    \includegraphics[width=6cm]{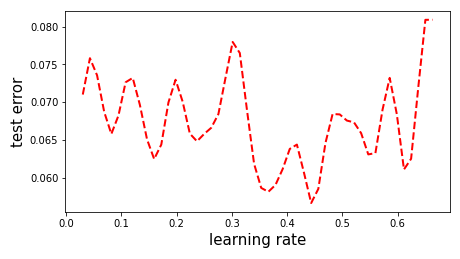}
    \caption{Test error of a CNN-classifier as a function of learning rate.}
    \label{fig:landscape}
\end{wrapfigure} 

In practice, many hyperparameters need to be chosen from continuous spaces. An important example is the learning rate; See Figure \ref{fig:landscape} for an illustration.  
From Figure \ref{fig:landscape} we observe: (1) the choice of hyperparameters has a great impact on model performance; (2) the objective function is continuous, but not well-behaving. Similar hyperparameters include weight parameters, noise schedules in stochastic models, and so on. There have been a lot of methods developed to tackle these problems, model-based \citep{10.5555/3104322.3104451,shahriari2015taking,falkner2018bohb,huang2022improving} or model-free  \citep{wu2011experiments,bergstra2012random,jamieson2016non,li2017hyperband}. However, existing model-based methods lack theoretical guarantees, unless harsh conditions are imposed; model-free methods seldomly search over the hyperparameter space adaptively, which may lead to inefficiency and worse final performance. Therefore, we need a method that (1) takes advantage of the continuity of objective function to guide searching; (2) has better theoretical guarantees than model-based methods.


To this end, we develop a new bandit-based approach for continuous HPO problems. When modeling HPO as a bandit problem, the hyperparameter configuration corresponds to arm, and the learning output corresponds to loss or reward. In this work, we formulate continuous HPO as a pure-exploration Lipschitz bandit problem, where the search space contains infinite arms, and the aim is to find the arm with minimal loss. We assume that the loss function is Lipschitz. This assumption captures the fact that closer hyperparameters tend to have similar losses, and we take advantage of it to design efficient HPO algorithms.

We propose Batched Lipschitz Exploration (BLiE) algorithm to solve this pure exploration problem. BLiE adaptively learns the landscape of objective function, and automatically assign more budget to promising hyperparameters. To sum up, BLiE has three advantages:

    $\bullet$ BLiE is model-free and only has one continuity assumption. Thus, our theoretical guarantees work for a wide range of problems. As a comparison, \cite{parker2020provably} assume the objective function is almost surely continuously differentiable and its derivative satisfies Lipschitz assumptions for getting the regret bound. However, these conditions are too complex to meet or verify.
    
    $\bullet$ BLiE takes advantage of the Lipschitz continuity of objective function to guide sampling, so it is more likely to find the best hpyerparameter. Theoretically speaking, by simple-regret analysis, we show that performance of BLiE is better than random-search-based methods \citep[e.g.,][]{bergstra2012random, jamieson2016non, li2017hyperband} when the HPO task is hard.
    
    $\bullet$ BLiE is suitable for batched feedback setting, and the decision-making process only needs very few data communications. Therefore, BLiE can naturally work in parallel.

Experimental results from different machine learning tasks show the superior performance of BLiE. Furthermore, we apply BLiE to noise scheduling in diffusion models. The BLiE schedule has competitive sample quality by using very few diffusion steps, and thus significantly improves sampling speed without using additional speeding up techniques.

\section{Related Works}

\textbf{Hyperparameter Optimization:} In recent years, the surging need of hyperparameter optimization algorithms from deep learning has motivated a larger cluster of researches. See \cite{feurer2019hyperparameter} for a recent exposition. To name a few, grid search \citep{wu2011experiments} and random search \citep{bergstra2012random} are now considered two standard benchmark methods. Inspired by biological findings, population-based methods have also been used for HPO tasks \citep{hansen2016cma,jaderberg2017population}. 
Another line of research is the model-based methods. In such methods, a model fitted on past observation is built, and subsequent hyperparameter trials are selected based on this model. Examples include Bayesian optimization algorithms with different surrogate \citep[e.g.,][]{10.5555/3104322.3104451,shahriari2015taking} 
, and tree-based methods \citep{hutter2011sequential,bergstra2011algorithms,10.1145/3412815.3416885}. More recently, HPO methods that explicitly model the training nature of neural networks have been invented. In these methods, obtaining a more accurate test/validation error requires higher training expenses. Such methods include multi-fidelity (Bayesian) optimization \citep{forrester2007multi,kandasamy2017multi,song2019general}, where feedback at finer fidelity are more accurate. Other methods that incorporate training budget include Successive Halving \citep{jamieson2016non}, Hyperband \citep{li2017hyperband}, and BOHB \citep{falkner2018bohb}. In particular, \cite{huang2022improving} designed a special multi-fidelity algorithm for Bayesian optimization. They proposed a special training data collection strategy for getting better estimation in Bayesian models. A better model can lead to better search area of hyperparameters. 

\textbf{Pure Exploration in Multi-Armed Bandits:} Another line of related works is pure exploration bandits, where the goal is to minimize the simple regret, or the gap between the optimal arm and the output one. \cite{bubeck2011pure} gave upper and lower bounds under the stochastic setting. \cite{jamieson2016non} and \cite{li2017hyperband} extended the problem to the non-stochastic setting. \cite{carpentier2015simple} studied pure exploration bandits with infinitely many arms, where the means of arms are drawn from a distribution $F$. \cite{even2006action} and \cite{mannor2004sample} studied a related setting, where the aim is to output an $\varepsilon$-optimal arm using as little budget $T$ as possible.

\textbf{Our advantages:} This paper studies HPO from a Lispchitz bandit perspective, and properly incorporate the training budget considerations into the Lipschitz best arm identification framework. Such HPO algorithms, to the best of our knowledge, have not been covered by existing works. As discussed in the introduction, Lipschitzness can better capture the loss landscape of the hyperparameters than existing setups. In addition, our algorithm is parallelizable, since the training feedback does not need to be frequently collected. Our method leverages virtues of both Lipschitz bandits and batched bandits. See Appendix \ref{app:related} for more related works.

\section{Preliminaries: Pure Exploration Lipschitz Bandits with Batched Feedback}



For continuous HPO tasks, the candidate hyperparameters are gathered into a compact subset of $\R^d$. When modeling HPO as pure exploration bandit problem, the arm set $\mathcal{X}\in\R^d$ corresponds to the set of hyperparameters, and pulling an arm corresponds to training the model.
Assigning budget $n$ to arm $x\in\mathcal{X}$ means training the model with $n$ units of resources (e.g., iterations), after which we receive a loss $\ell(x,n)$. 
Similar to existing bandit-modeling of HPO \citep[e.g.,][]{jamieson2016non, li2017hyperband}, we assume that for any $x$ there exists a \emph{limit loss} $\mu(x)=\lim_{n\to\infty}\ell(x,n)$ and we define the optimal limiting loss as $\mu^*=\min_{x\in\mathcal{X}}\mu(x)$. We also make the following assumption.

\begin{assumption}\label{ass:hoeffding}
For any $x\in\mathcal{X}$ and $n \in \mathbb{N}_+$, the error sequence $\{\ell(x,n)\}_{n=1}^\infty$ is bounded by $|\ell(x,n)-\mu(x)|\leq n^{-\frac{1}{\beta}}$, for some $\beta>0$. 
\end{assumption}

Assumption \ref{ass:hoeffding} assumes a power-law decay of the gap between $\ell (x,n)$ and $\mu (x)$. This assumption resonates with the convergence rate of most gradient-based training algorithms. The goal of a pure-exploration bandit algorithm is to output an arm $\widetilde{x}^*$ with as small optimal gap $\Delta_{\widetilde{x}^*}:=\mu(\widetilde{x}^*)-\mu^*$ as possible. A general form of pure exploration bandits is in Algorithm \ref{alg:gen}.

      \begin{algorithm}[H]
	\caption{Pure Exploration Bandits} 
	\label{alg:gen} 
        \begin{algorithmic}[1]
		\STATE \textbf{Input.} Arm set $\mathcal{X}$; Total budget $T$.
  		\WHILE{remaining budget $T>0$}
		    \STATE Algorithm chooses arm $x$ and budget $n$.
		    \STATE Assign arm $x$ with budget $n$, and receive $\ell(x, n)$; $T\gets T-n$.
		\ENDWHILE
            \STATE Output an arm $\widetilde{x}^*$.
        \end{algorithmic}
      \end{algorithm}



\subsection{Lipschitz Bandits Model}
Now we expound our Lipschitz bandits setting. We would like to take advantage of the continuity of the objective function, as shown in Figure \ref{fig:landscape}. Also, we do not want to introduce parametric models. Therefore, we make the following assumption.

\begin{assumption}\label{ass:lip} 
The limiting loss $\mu(x)$ is $L$-Lipschitz with respect to the metric on $\mathcal{X}$, that is, $|\mu(x_1)-\mu(x_2)|\leq L\cdot\|x_1-x_2\|$, for any $x_1,\;x_2\in\mathcal{X}$.
\end{assumption} 

As discussed in the introduction and above, this assumption captures the behavior of many important hyperparameters. Before moving on to the next part, we put forward the following conventions. 
\begin{remark} 
    \label{rem:doubling} 
    As a convention, we focus on the metric space $ ([0,1]^d, \| \cdot \|_\infty) $.
\end{remark} 
Note that the restriction in Remark \ref{rem:doubling} does not sacrifice generality. By the Assouad's embedding theorem \citep{assouad1983plongements}, the (compact) doubling metric space $ \mathcal{X} $ can be embedded into a Euclidean space with some distortion of the metric; See \cite{pmlr-v119-wang20q} for more discussions in a machine learning context. Due to existence of such embedding, the metric space $ ([0,1]^d, \| \cdot \|_{\infty})  $, where metric balls are hypercubes, is sufficient for the purpose of our paper. For the rest of the paper, we will use hypercubes in algorithm design for simplicity, while our algorithmic idea generalizes to other doubling metric spaces. 



\subsection{Zooming Number and Zooming Dimension}\label{sec:zooming}

We use the zooming number and the zooming dimension \citep{kleinberg2008multi, bubeck2008tree, slivkins2011contextual} in our theoretical analysis. These are important concepts for bandits in metric spaces, and we explain them below.


Define the set of $r$-optimal arms as $ S (r) = \{ x \in \mathcal{X} : \Delta_x \le r \} $. For any $r=2^{-i}$, the decision space $[0,1]^d$ can be equally divided into $2^{di}$ cubes with edge length $r$, which we call \textit{standard cubes} (also referred to as dyadic cubes). 
The $r$-zooming number is defined as 
\begin{equation*}
	N_r := \#\{C: \text{$C$ is a standard cube with edge length $r$ and }\text{$C\subset S((8L+8)r)$}\}.
\end{equation*}
The zooming dimension is then defined as $d_z := \min \{  d\geq0: \exists a>0,\;N_r \le ar^{-d},\;\forall r=2^{-i} \text{ for $i \in \mathbb{N}$}\}$. Moreover, we define the zooming constant $C_z$ as 
$C_z=\min\{  a>0:\;N_r \le ar^{-d_z},\;\forall r=2^{-i} \text{ for $i \in \mathbb{N}$} \}$.

It is obvious that $d_z$ is upper bounded by ambient dimension $d$. In fact, zooming dimension $d_z$ can be significantly smaller than $d$ and can be zero. For a simple example, consider a problem with ambient dimension $d=1$ and expected reward function $\mu(x)=x$ for $0\leq x\leq1$, which satisfies Assumption \ref{ass:lip} with $L=1$. Then for any $r=2^{-i}$ with $i\geq4$, we have $S(16r)=[1-16r,1]$ and $N_r=16$. Therefore, for this problem the zooming dimension equals to $0$, with zooming constant $C_z=16$.

\subsection{Bandit Problems with Batched Feedback} 

The batched bandit problem is a trending topic in multi-armed bandit problems (See Appendix \ref{app:related}). 
In such problems, the observed losses are communicated to the agent in batches, and the decisions made by the algorithm depend only on information up to the previous batch. Algorithms with good performance for batched bandits also have advantages in HPO problems.
Since the policy  does not depend on observations from the same batch, hyperparameters belonging to the same batch can be trained in parallel.

In the bandit language, this feedback collecting scheme is called bandit with batched feedback. Here we define the batched feedback pattern formally. For a $T$-step game, the player determines a grid $\mathcal{T}=\{t_0,\cdots,t_B\}$ \emph{adaptively}, where $0=t_0<t_1<\cdots<t_B=T$ and $B\ll T$. During the game, loss observations are communicated to the player only at the grid points $t_1, \cdots, t_B$. As a consequence, for any time $t$ in the $j$-th batch, that is, $t_{j-1}<t\leq t_j$, the loss generated at time $t$ cannot be observed until time $t_j$, and the decision made at time $t$ depends only on losses up to time $t_{j-1}$. The determination of the grid $\mathcal{T}$ is adaptive in the sense that the player chooses each grid point $t_j\in\mathcal{T}$ based on the operations and observations up to the previous point $t_{j-1}$.

\section{Algorithm} \label{sec:alg}

\begin{wrapfigure}{r}{0.5\textwidth} 
	\centering
	\subfigure{
		\includegraphics[width=2cm]{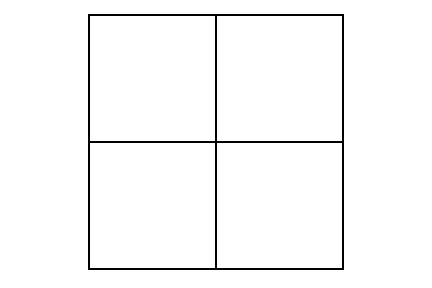}
	}\hspace{-6.75mm}
	\subfigure{
		\includegraphics[width=2cm]{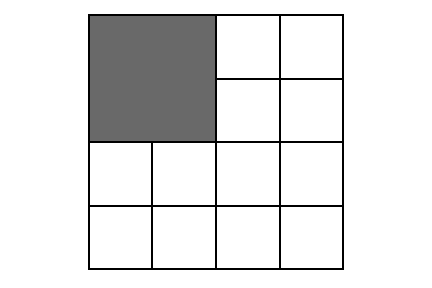}
	}\hspace{-6.75mm}
	\subfigure{
		\includegraphics[width=2cm]{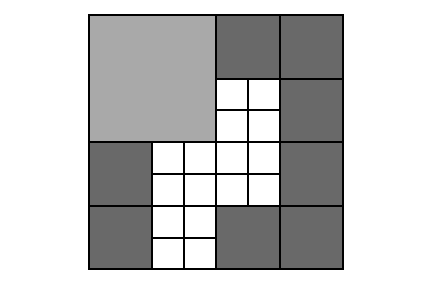}
	}\hspace{-6.75mm}
	\subfigure{
		\includegraphics[width=2cm]{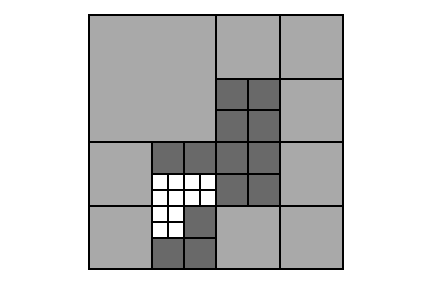}
	}
	\caption{Partition and elimination process of a BLiE run. The $i$-th subfigure shows the pattern before the $i$-th batch. Dark gray cubes are those eliminated in the most recent batch, while the light gray ones are those eliminated in earlier batches.} 
	\label{f:visual}
\end{wrapfigure} 

In this section, we propose Batched Lipschitz Exploration (BLiE) algorithm to solve pure-exploration Lipschitz bandits with batched feedback. The main policies of BLiE are inspired by \cite{fenglipschitz}. However, the novel analysis in this work shows that BLiE can efficiently identify good arms by using very few data communications, and has better performance than uniform search and random-search-based algorithms.


In a batched feedback setting, the agent’s knowledge does not build up within each batch. Therefore, a `uniform’ type algorithm is naturally suitable for such problems. Based on this intuition, BLiE treats decisions made in the same batch equally. More specifically, it works in the following four steps in each batch $m$. 1. Construct a collection $\mathcal{A}_m$ of cubes, where each cube is a subset of $\mathcal{X}$ and has \emph{the same} edge length $r_m$. Assign \emph{the same} budget to all the cubes in $\mathcal{A}_m$; 2. Receive the observed loss of each cube at the end of the batch; 3. Eliminate cubes with high losses; 4. Further partition the remaining cubes to smaller subcubes and collect these subcubes to construct $\mathcal{A}_{m+1}$. The learning process of BLiE is summarized in Algorithm \ref{BLiN}. Moreover, we present a visualization of partition and elimination steps of a real BLiE run in Figure \ref{f:visual}, where the arm space is $[0, 1]^2$.

\begin{algorithm}[htb]
	\caption{Batched Lipschitz Exploration (BLiE)} 
	\label{BLiN} 
	\begin{algorithmic}[1]  
		\STATE \textbf{Input.} Arm set $\mathcal{X}=[0,1]^d$; Total budget $T$; $\alpha$ and $\beta$.
		\STATE \textbf{Initialization} Edge-length sequence $\{r_m\}_{m\in\mathbb{N}^+}$; The first grid point $t_0=0$; Equally partition $\mathcal{X}$ to $r_1^{-d}$ subcubes with edge length $r_1$ and define $\mathcal{A}_{1}$ as the collection of these subcubes.
		
		\FOR{$m=1,2,\cdots$}
			\STATE For each $C\in\mathcal{A}_m$, randomly choose\footnotemark{} an arm $x_C$ and evaluate $x_C$ with budget $n_m={r_m^{-\beta}}$.
			\STATE Receive the loss $\ell(x_C,n_m)$ for each cube $C\in\mathcal{A}_m$.
			Find  $\ell_m^{\min}=\min_{C\in\mathcal{A}_m}\ell(x_C,n_m)$.
			\STATE For each cube $C\in\mathcal{A}_m$, eliminate $C$ if $\ell(x_C,n_m)-\ell_m^{\min}>\alpha r_m$. Let $\mathcal{A}_m^+$ be set of cubes not eliminated.
			\STATE \label{cond:break}Compute $t_{m+1}=t_m+(r_{m}/r_{m+1})^d\cdot|\mathcal{A}_m^+|\cdot n_{m+1}$. If $t_{m+1}\geq T$ then define $\mathcal{X}_c=\{x_C:\;C\in\mathcal{A}_m^+\}$ and \textbf{break}.
			\STATE Equally partition each cube in $\mathcal{A}_m^+$ into $\left(r_m/r_{m+1}\right)^d$ subcubes with edge length $r_{m+1}$ and define $\mathcal{A}_{m+1}$ as the collection of these subcubes. 
		\ENDFOR 
	\STATE Assign budget $n_f$ to each arm $x\in\mathcal{X}_c$ uniformly, so that the total budget is used up. 
	\STATE\textbf{Output}  $\widetilde{x}^*=\mathop{\arg\min}_{x\in\mathcal{X}_C}\ell(x, n_m+n_f)$. 
	\end{algorithmic} 
\end{algorithm}\footnotetext{One can arbitrarily pick $x_C$ from cube $C$. In practice, we pick $x_C$ uniformly at random.}

\section{Theoretical Results}

As mentioned above, there are three theoretical contributions in this work: 1. We provide simple regret upper bound of BLiE algorithm; 2. We show that BLiE requires very few rounds of data communications; 3. We also develop simple regret lower bounds for uniform search and random-search-based algorithms, which demonstrate that BLiE has better performance.

\subsection{Simple Regret Upper Bound of BLiE}\label{sec:log}
This section gives the simple regret upper bound of BLiE with Doubling Edge-length Sequence $r_m=2^{-m}$. Doubling Sequence is also the edge-length sequence used in our experiments.

\begin{theorem}\label{thm:blin-upper} 
If Assumption \ref{ass:hoeffding} and \ref{ass:lip} are satisfied, then output arm $\widetilde{x}^*$ of BLiE algorithm with total budget $T$, edge-length sequence $r_m=2^{-m}$, $\alpha=2L+2$ and $\beta$ satisfies
\begin{equation}\label{upp-simple}
    \Delta_{\widetilde{x}^*}\leq c\cdot T^{-\frac{1}{d_z+\beta}},
\end{equation}
where $d_z$ is the zooming dimension and $c$ is a constant. In addition, BLiE needs no more than $\frac{1}{d_z+\beta}\log T$ batches to achieve this simple regret.
\end{theorem} 

Our analysis needs the following lemma, which shows that BLiE can gradually identify areas with small simple regret. The proofs of Theorem \ref{thm:blin-upper} and Lemma \ref{lem:eli} are in Appendix \ref{app:thm1} and \ref{app:lem1}.

\begin{lemma}\label{lem:eli}
    For any $m\geq1$ and $x\in \cup_{C\in\mathcal{A}_m}C$, we have $\Delta_x\leq(4L+4)r_{m-1}$.
\end{lemma}

\subsection{Achieving Better Communication Bound}\label{sec:loglog}

The communication bound can be improved without causing worse simple regret. To achieve this, we use the following edge length sequence.

\begin{definition}
    \label{def:ace}
    For a problem with ambient dimension $d$, zooming dimension $d_z$ and time horizon $T$, we denote $c_1=\frac{d_z+\beta-1}{(d_z+\beta)(d+\beta)}\log T$ and $c_{i+1} = {\eta c_i }$ for any $i\ge 1$, where $\eta=\frac{d+1-d_z}{d+\beta}$. Then we let $\alpha_n=\lfloor\sum_{i=1}^nc_i\rfloor$, $\beta_n=\lceil\sum_{i=1}^nc_i\rceil$, and define\footnote{To simplify the notation, in this subsection we assume $\{r_m\}$ is strictly decreasing. See Appendix \ref{app:race} for the version without this assumption.} Appropriate Combined Edge-length Sequence (ACE Sequence) $\{r_m\}_{m\in\mathbb{N}}$ as $r_m=2^{-\alpha_k}$ for $m=2k-1$ and $r_m=2^{-\beta_k}$ for $m=2k$.
\end{definition}
Then we show that BLiE with ACE Sequence can achieve the same simple regret using only $\mathcal{O}(\log\log T)$ batches.

\begin{theorem}\label{thm:ace}

If Assumption \ref{ass:hoeffding} and \ref{ass:lip} are satisfied, then output arm $\widetilde{x}^*$ of BLiE algorithm with total budget $T$, ACE Sequence $\{r_m\}$, $\alpha=2L+2$ and $\beta$ satisfies 
\[\Delta_{\widetilde{x}^*}\leq c\cdot T^{-\frac{1}{d_z+\beta}},\]
where $d_z$ is the zooming dimension and $c$ is a constant. In addition, BLiE needs $\mathcal{O}(\log\log T)$ batches to achieve this simple regret.
\end{theorem}
\begin{proof}
We let $B$ be the total number of batches of the BLiE run ($B-1$ batches in the for-loop and $1$ clean-up batch), and $N_m$ be the total budget of batch $m$. In the following analysis, we bound $N_m$ for $m=2k-1$ and $m=2k$ separately, and then obtain the simple regret upper bound.

Firstly, we consider the case $m=2k-1$. For convenience, we let $\widetilde{r}_k=2^{-\sum_{i=1}^kc_i}$, and thus we have $\widetilde{r}_{k-1}\geq r_{m-1}\geq r_m\geq \widetilde{r}_{k}$. Recall that $\mathcal{A}_{m-1}^+$ is set of cubes not eliminated in batch $m-1$, and $\cup_{C\in\mathcal{A}_{m-1}^+}C=\cup_{C^\prime\in\mathcal{A}_m}C^\prime$. Lemma \ref{lem:eli} implies that each cube in $\mathcal{A}_{m-1}^+$ is a subset of $S((4L+4)r_{m-1})$, and thus $\left|\mathcal{A}_{m-1}^+\right|\leq N_{r_{m-1}}\leq C_z\cdot r_{m-1}^{-d_z}$. The total budget of batch $m$ is
\begin{align*}
    N_m=|\mathcal{A}_m|\cdot n_m = \left(\frac{r_{m-1}}{r_m}\right)^d\left|\mathcal{A}_{m-1}^+\right|\cdot n_m
    \leq C_z\cdot\frac{r_{m-1}^{d+1-d_z}\cdot r_m^{-d-\beta}}{r_{m-1}} 
    \leq C_z\cdot\frac{\widetilde{r}_{k-1}^{d+1-d_z}\cdot\widetilde{r}_k^{-d-\beta}}{r_{m-1}}.
\end{align*}
For the numerator, we have
\begin{align*}
    \widetilde{r}_{k-1}^{d+1-d_z}\cdot\widetilde{r}_k^{-d-\beta}
    =2^{-\left(\sum_{i=1}^{k-1}c_i\right)(d+1-d_z)+\left(\sum_{i=1}^{k}c_i\right)(d+\beta)}=2^{\left(\sum_{i=1}^{k-1}c_i\right)(d_z+\beta-1)+c_m(d+\beta)}.
\end{align*}
Define $C_m=\left(\sum_{i=1}^{m-1}c_i\right)(d_z+\beta-1)+c_m(d+\beta)$. Since $c_m=c_{m-1}\cdot\frac{d+1-d_z}{d+\beta}$, calculation shows that $C_m = (\sum_{i=1}^{m-2} c_i) (d_z+\beta-1)  + c_{m-1}(d+\beta)+ c_{m-1}(d_z+\beta-1-d-\beta)+c_m(d+\beta)=C_{m-1}$. Thus for any $m$, we have $C_M=C_1=\frac{d_z+\beta-1}{d_z+\beta}\log T$. Hence,
\begin{equation*}
    N_m\leq C_z\cdot2^{\frac{d_z+\beta-1}{d_z+\beta}\log T}/r_{m-1}=C_z\cdot T^{\frac{d_z+\beta-1}{d_z+\beta}}/r_{m-1}.
\end{equation*}

Secondly, we consider the case $m=2k$. Lemma \ref{lem:eli} implies that each cube in $\mathcal{A}_m$ is a subset of $S((8L+8)r_m)$. Similar argument to the first case shows that $\left|\mathcal{A}_m\right|\leq C_z\cdot r_m^{-d_z}$ and $N_m\leq C_z\cdot r_m^{-(d_z+\beta)}$.

Line 7 of Algorithm \ref{BLiN} ensures that the sum of budgets of full $B$ batches is greater than $T$. Therefore, combining the above two cases, we have
\begin{align}\label{upp-budget-ace}
    T\leq&\sum_{\substack{m=2k-1,\;m\leq B }}C_z \cdot\frac{T^{\frac{d_z+\beta-1}{d_z+\beta}}}{r_{m-1}}+\sum_{\substack{m=2k,\;m\leq B}}C_z\cdot \frac{1}{r_m^{d_z+\beta}}.
\end{align}

The rounding step in Definition \ref{def:ace} yields that $r_m\leq\frac{r_{m-1}}{2}$ for any $m$. If $B$ is odd, from (\ref{upp-budget-ace}) we have $T\leq2C_z\cdot T^{\frac{d_z+\beta-1}{d_z+\beta}}\cdot\frac{1}{r_{B-1}}+2C_z\cdot\frac{1}{r_{B-1}^{d_z+\beta}}$.  We set $c_r=\max\{(4C_z)^\frac{1}{d_z+\beta}, 4C_z\}$, then $\frac{1}{c_r^{d_z+\beta}}+\frac{1}{c_r}\leq\frac{1}{2C_z}$ and the above inequality implies that $r_{B-1}\leq c_r\cdot T^{-\frac{1}{d_z+\beta}}$. Lemma \ref{lem:eli} shows that $\Delta_x\leq(4L+4)r_{B-1}$ for any $x\in\cup_{C\in\mathcal{A}_B}C$, so we have $\Delta_{\widetilde{x}^*}\leq c\cdot T^{-\frac{1}{d_z+\beta}}$, where $c=(8L+8)\cdot\max\{c_r, 1\}$. If $B$ is even, similar arguments also yield that $\Delta_{\widetilde{x}^*}\leq c\cdot T^{-\frac{1}{d_z+\beta}}$. See Appendix \ref{app:race} for details.

Finally, we consider the communication bound. For any $B^*$, $\widetilde{r}_{B^*}=2^{-\sum_{i=1}^{B^*}c_i}=2^{-c_1\cdot\frac{1-\eta^{B^*}}{1-\eta}}=T^{-\frac{1}{d_z+\beta}}\cdot T^{\frac{\eta^{B^*}}{d_z+\beta}}$. Then by choosing $B^*\geq\frac{\log\log T-\log(d_z+\beta)}{\log\frac{d+\beta}{d+1-d_z}}$, we have $\frac{\eta^{B^*}}{d_z+\beta}\leq\log T$ and $\widetilde{r}_{B^*}\leq T^{-\frac{1}{d_z+\beta}}$. Definition \ref{def:ace} shows that $r_m<\widetilde{r}_{B^*}$ for any $m\geq2B^*$. Thus, no more than $\hat{B}=\frac{2\log\log T}{\log\frac{d+\beta}{d+1-d_z}} + 1$ batches are needed to achieve $r_{\hat{B}-1}\leq T^{-\frac{1}{d_z+\beta}}$ and $\Delta_{\widetilde{x}^*}\leq (4L+4)T^{-\frac{1}{d_z+\beta}}\leq c\cdot T^{-\frac{1}{d_z+\beta}}$.
\end{proof}

\subsection{Lower Bound for Uniform Search}
Now we derive the theoretical performance of uniform search strategy for pure-exploration Lipschitz bandits. The following theorem provides lower bound of the resulting optimal gap for uniform search. Recall that simple regret upper bound of BLiE is $\mathcal{O}\left(T^{-1/({d_z+\beta})}\right)$, this result yields that theoretical performance of uniform search is worse than BLiE. Pseudocode of uniform search strategy and proof of Theorem \ref{thm:lower-us} are presented in Appendix \ref{app:us}.

\begin{theorem}\label{thm:lower-us}
For any total budget $T$, dimension $d$ and grid length $r$, there exists an instance with $d_z=0$ such that the uniform search strategy returns an arm $\widetilde{x}^*$ with optimal gap $\E\Delta_{\widetilde{x}^*}\geq \frac{1}{2}T^{-\frac{1}{d+\beta}}$.
\end{theorem}

\subsection{Lower Bound for Random-Search-Based Algorithms}

The following theorem provides lower bound of optimal gap for random search strategy, that is, randomly sample $N$ arms and recommend one of them according to a certain policy. Pseudocode of random search strategy and proof of Theorem \ref{thm:lower-hb} are presented in Appendix \ref{app:lower-hb}.



\begin{theorem}\label{thm:lower-hb}
For any total budget $T$, dimension $d$, zooming dimension $d_z$ and number of selected arms $N$, there exists an instance such that any random-search-based algorithm returns an arm $\widetilde{x}^*$ with optimal gap $\E\Delta_{\widetilde{x}^*}\geq c\cdot T^{-\frac{1}{d-d_z}}$, where $c$ is a constant.
\end{theorem}



In Section \ref{sec:log} and \ref{sec:loglog}, we show that the simple regret upper bound of BLiE is of order $\mathcal{O}\left(T^{-1/({d_z+\beta})}\right)$. Therefore, when 
$d_z\leq\frac{d-\beta}{2}$, BLiE outperforms random-search-based algorithms. 
Smaller $d_z$ means that the near-optimal region is smaller, and thus finding a sufficient good arm is harder. Consequently, based on concepts from Lipschitz bandits, we show that BLiE outperforms random search when the problem is hard. Note that the above bound is valid for any random-search-based algorithm, thus including Hyperband. In the next subsection, 
we further make a comparison of the two algorithms and explain the reasons for the superiority of BLiE.

\subsection{Comparison with Hyperband under Lipschitz Bandits Setting}

\citet{li2017hyperband} parameterize the CDF of $\mu(x)$ as $F(\nu)\simeq(\nu-\mu^*)^\gamma$. Now we calculate $\gamma$ under the Lipschitz bandits setting. For any $r>0$, we let $\nu=\mu^*+r$. Then $F(\nu)=\frac{\m(\{x:\;\mu(x)-\mu^*\leq r\})}{\m([0, 1]^d)}=\m(\{x:\Delta_x\leq r\})$, where $\m(A)$ denotes the  measure of set $A$. From the definition of zooming number, we know the set $\{x:\Delta_x\leq r\}$ is packed by $N_{\frac{r}{8L+8}}\leq(8L+8)^{d_z}C_z\cdot r^{-d_z}$ cubes with edge length $\frac{r}{8L+8}$.
If this inequality is tight, then we have $\m(\{x:\Delta_x\leq r\})\approx N_{\frac{r}{8L+8}}\cdot \left(\frac{r}{8L+8}\right)^d\approx c\cdot r^{d-d_z}$ and $F(\nu)\approx c\cdot r^{d-d_z}=c\cdot(\nu-\mu^*)^{d-d_z}$, where $c$ is a constant. Thus, we obtain an approximate correspondence $\gamma=d-d_z$. Then Theorem 5 in \cite{li2017hyperband} shows that the output arm of Hyperband satisfies $\Delta_{\widetilde{x}^*}\leq T^{\max\left\{-\frac{1}{d-d_z},-\frac{1}{\beta}\right\}}$.

Indeed, theoretical success of Hyperband heavily relies on hitting a good arm in the random sample procedure of a certain SuccessiveHalving subroutine. When the near-optimal region is small, the simple regret bound of Hyperband may get worse or even break.
As a comparison, BLiE only needs to identify and eliminate the sub-optimal region. Thus theoretically, BLiE outperforms Hyperband at least in the following two aspects: 1. As is shown above, when the zooming dimension is small (or equivalently, the near-optimal region is small), the output optimal gap of BLiE is better than Hyperband; 2. BLiE only needs an upper bound of the volume of the near-optimal region. More specifically, upper bound (\ref{upp-simple}) holds when covering number of near-optimal region $N_r$ is upper bounded by $r^{-d_z}$. As a comparison, Hyperband needs an additional assumption that $F(\nu)\gtrsim(\nu-\mu^*)^\gamma$ to ensure the near optimal-region is not too small.

\section{Experiments}

This section provides empirical comparison of BLiE with existing HPO methods including Hyperband (HB), SuccessiveHalving (SH), Random Search (RS) and Tree-structured Parzen Estimator (TPE). The results demonstrate the superior performance of BLiE. Also, we apply BLiE to noise scheduling task of diffusion models. Compared with the standard linear schedule, BLiE schedule only needs very few diffusion steps, and thus greatly improves the sampling speed. Results of toy example are averaged over $256$ runs, and results in Section \ref{sec:exp1} and \ref{sec:exp2} are averaged over $32$ runs.

\begin{figure*}[ht]
    \centering
	\subfigure[Toy Example]{
        \includegraphics[width=4cm]{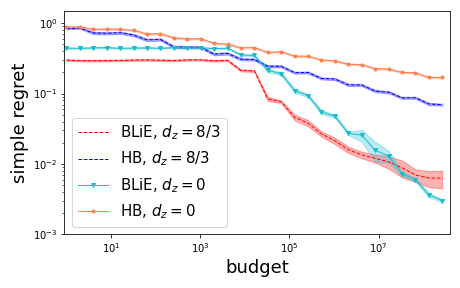}
		\label{fig:toy}
	} 
	\subfigure[MNIST]{
		\includegraphics[width=4cm]{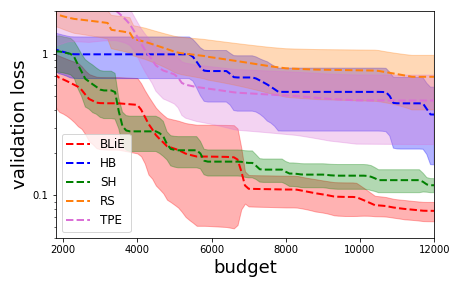}
		\label{fig:mnist}
	} 
	\subfigure[CIFAR-10]{
		\includegraphics[width=4cm]{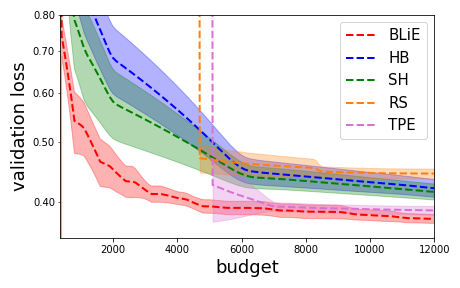}
		\label{fig:cifar}
	}
    \caption{HPO processes for different tasks. Figure \ref{fig:toy} shows results of toy example with different limit losses. Figure \ref{fig:mnist} and \ref{fig:cifar} show results of tuning optimizer for neural-network classifiers on MNIST and CIFAR-10.}
    \label{fig:adam}
\end{figure*}

\subsection{Toy Example}


In this experiment, we investigate the performance of BLiE in a high-dimensional toy example. We also run Hyperband (HB) as a comparison. The arm space is $\mathcal{X}=[0, 1]^8$. In order to compare the two algorithms under different $d_z$, we consider two different limit loss functions. Given an arm $x\in\mathcal{X}$, we define $\mu_1(x)=\|x\|_\infty$ ($d_z=0$) and $\mu_2(x)=\|x\|_\infty^{1.5}$ ($d_z=8/3$). For any limit loss $\mu$, assign $1$ budget to arm $x$ corresponds to sample a Gaussian random variable $Y_{x,i}$ with mean equals to $\mu(x)$, and $\ell(x,n)=\frac{\sum_{i=1}^nY_{x, i}}{n}$. Consequently, Assumption \ref{ass:lip} is satisfied with $L=1$, and Assumption \ref{ass:hoeffding} is asymptotically satisfied with $\beta=2$. 

We run BLiE and HB with total budget $T=2^{28}$, and report the simple regret $\Delta_{\widetilde{x}^*}=\mu(\widetilde{x}^*)-\mu(x^*)$ in Figure \ref{fig:toy}.
We have two observations from Figure \ref{fig:toy}. First, for both limit losses, the simple regrets of BLiE (red and light blue lines) are smaller than HB (orange and dark blue lines). Second, final performance of BLiE on $\mu_1$ (light blue line) is better than $\mu_2$ (red line), while performance of HB on $\mu_2$ (dark blue line) is better than $\mu_1$ (orange line). 
We prove in Theorem \ref{thm:blin-upper} that the simple regret upper bound of BLiE is $\mathcal{O}\left(T^{-1/({d_z+\beta})}\right)$, which means that BLiE benefits from smaller $d_z$. On the other hand, we prove in Theorem \ref{thm:lower-hb} that the simple regret lower bound of HB is $\Omega\left(T^{-1/({d-d_z})}\right)$, which means that HB suffers from smaller $d_z$. In this experiment, $\mu_1$ has $d_z=0$ and $\mu_2$ has $d_z=8/3$. These results match our theoretical analysis and show that our simple regret bounds are tight.

\subsection{Tuning Optimizer for Neural Networks}\label{sec:exp1}

\begin{table*}[th]
    \centering
    \caption{{Test accuracy of classification tasks. The models are trained with hyperparameters output by the five methods.}}
    \begin{tabular}{ccccccc}
         \toprule
         \multicolumn{2}{c}{method} & BLiE & HB & SH & RS & TPE\\
         \midrule
         \multirow{2}*{Acc (std)} & MNIST & \textbf{96.3} (0.4) & 95.3 (0.6) & 95.7 (0.7) & 95.2 (0.5) & 94.8 (0.6)\\
         \cmidrule{2-7}
          & CIFAR-10 & \textbf{91.2} (0.1) & 87.7 (3.0) & 88.3 (2.7) & 87.2 (2.6) & 90.4 (0.2)\\
         \bottomrule
    \end{tabular}
    \label{tab:accuracy}
\end{table*}

In this experiment, we apply BLiE to tune the Adam Optimizer \citep{kingma2014adam} for two classification tasks. The hyperparameter set consists of learning rate $lr$, weights $\beta_1$ and $\beta_2$. We take experiments on two datasets: MNIST and CIFAR-10. For MNIST, one unit of resource corresponds to one mini-batch training. For CIFAR-10, one unit of resource corresponds to $60$ mini-batch training.
We set the parameters of BLiE as $\alpha = 0.01$ and $\beta=2.5$ for all experiments in Section \ref{sec:exp1} and Section \ref{sec:exp2}.

For the model architecture, we use a two-layer CNN in the MNIST task, and Resnet18 \citep{he2016deep} in the CIFAR-10 task. We choose relatively simple models because our purpose is to compare the performance of HPO algorithms rather than obtain state-of-the-art accuracy, and this can save computational resources. We run BLiE, HB, SH, RS and TPE with total budget $12000$, and report the results in Figure \ref{fig:adam}. The results show that BLiE can not only find good hyperparameters faster, but also output better solutions at the end. We use the found hyperparameters to train both models, and report the final test accuracy in Table \ref{tab:accuracy}. This result shows that the model trained with hyperparameters output by BLiE has the best accuracy.


\subsection{Improving Noise Schedule of Diffusion Models}\label{sec:exp2}

Diffusion probabilistic models (DPM) \citep{Ho2020DenoisingDP,Song2020ScoreBasedGM,Dhariwal2021DiffusionMB} is a powerful family of generative models. DPMs have achieved state-of-the-art performance on various applications including image generation \citep{Rombach2021HighResolutionIS}, audios or videos generation \citep{Kong2020DiffWaveAV,Ho2022VideoDM}, drug designs \citep{Hoogeboom2022EquivariantDF}, and so on. 

Following \cite{Song2020ScoreBasedGM}, 
a diffusion model diffuses data distribution with a forward diffusion SDE and generates samples with a reverse SDE. Moreover, a neural score network is used to approximate marginal score functions of forward diffusion, which are needed in the reverse SDE.
The Variance Preserving (VP) forward diffusion \citep{Song2020ScoreBasedGM} is favored as an Ornstein--Uhlenbeck-type diffusion SDE across the literature, which takes the form $dX_t = -\frac{1}{2}\beta(t)X_tdt  + \frac{1}{2}\sqrt{\beta(t)}dW_t$.
Under certain conditions on $\beta (t)$, such diffusion enables arbitrary initial distribution to converge to multivariate Gaussian with sufficiently large $t$. Here $\beta(t)>0$ is called noise schedule of VP forward diffusion. The noise schedule is an important functional hyperparameter of VP diffusion which can influence both the learning efficiency and generative performance of DPMs. The pioneering work \citep{Ho2020DenoisingDP} proposed a linear learning schedule, for which the noise schedule $\beta(t)$ grows linearly with a start $\beta(0)=10^{-4}$ and an end $\beta(1)=2\times 10^{-2}$ with $t\in [0,1]$. 

\begin{figure*}[ht]
	\centering
	\subfigure[T=1000]{
		\includegraphics[width=4cm]{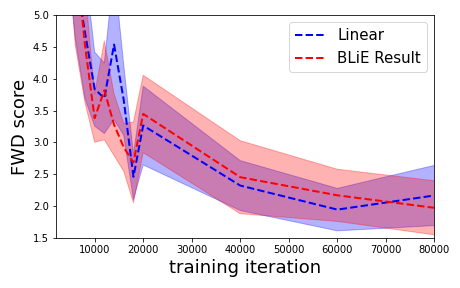}
		\label{f:fid-1000}
	}
	\subfigure[T=200]{
		\includegraphics[width=4cm]{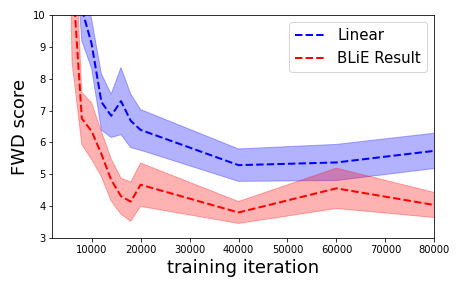}
		\label{f:fid-200}
	}
	\subfigure[T=100]{
		\includegraphics[width=4cm]{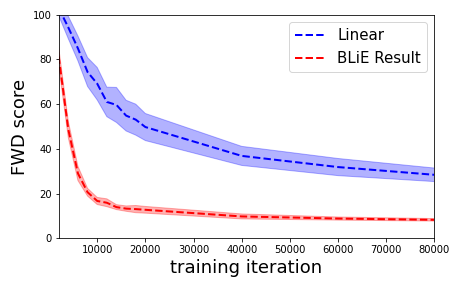}
		\label{f:fid-100}
	}
	\caption{FWD score of DDPM with different diffusion steps $T$.} 
	\label{f:fid}
\end{figure*}

In spite of impressive generative performance on various forms of data, the major drawback of DPMs is the slow sampling speed compared to other generative models such as GANs, VAEs, or Normalizing Flows. Usually, more than 1k diffusion steps are needed for the best performance of a diffusion model, and thus the same number of neural function evaluations (NFEs) are needed in sampling. Advanced simulation techniques of sampling SDE or ODE in order to reduce the NFEs of DPMs were intensively studied in recent works \citep{Karras2022ElucidatingTD,Bao2022AnalyticDPMAA}.
However, the state-of-the-art sampling technique still needs more than 30+ NFEs to achieve competitive generative performance on image datasets in terms of Fretchet Inception Distance (FID) \citep{Heusel2017GANsTB}.






In this experiment, we use BLiE to search for noise schedules with fewer diffusion steps $T$. Our result shows that by using only $T=100$ or $200$ diffusion steps, BLiE schedule achieves competitive sample quality with a linear schedule using $T=1000$ diffusion steps. It means that the sampling speed can be greatly improved without using any additional speeding-up technique. Recent works point out that DPMs' forward diffusion can be roughly divided into three stages \citep{deja2022analyzing}.
Based on such results, we consider searching a three-stage piece-wise linear schedule and let the two knots be the hyperparameters. More precisely, we search for four hyperparameters $t_a$, $\beta_a$, $t_b$, $\beta_b$ such that $0<t_a<t_b<1$ and $10^{-4}<\beta_a<\beta_b<2\times10^{-2}$, and the corresponding noise schedule $\beta(t)$ is a piece-wise linear function with start 
point $\beta(0)=10^{-4}$, end point $\beta(1)=2\times10^{-2}$, and knots $\beta(t_a)=\beta_a$, $\beta(t_b)=\beta_b$.

\begin{wrapfigure}{r}{0.5\textwidth}
    \centering
    \includegraphics[width=5cm]{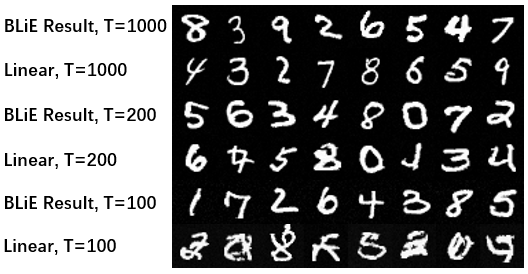}
    \caption{MNIST samples generated using different noise schedules and diffusion steps.}
    \label{fig:mnist_photo}
\end{wrapfigure}   

The experiment aims to demonstrate the feasibility of the proposed HPO algorithm for improving diffusion models, so we conduct the diffusion models experiments on the MNIST dataset to save computational costs. Because the MNIST dataset is grayscale, we mimic the calculation of FID on colored images by replacing the 
inception-v3 model with our pre-trained wide-resnet \citep{Zagoruyko2016WideRN} on MNIST, so we named this metric the Frechet Wideresnset Distance (FWD). 
Figure \ref{f:fid} presents the FWD curves along training iterations for 1000 (default), 200, and 100 diffusion steps respectively. As shown in the figure, red curves represent BLiE schedule, and blue curves represent the default linear schedule. 
The fewer diffusion steps we take, the more advantageous the BLiE schedule is compared to the default schedule. Samples generated using different noise schedules and diffusion steps are shown in Figure \ref{fig:mnist_photo}. This figure also shows that sample quality of BLiE schedule is better when $T$ is small. 
Also, we find that the marginal distribution of VP diffusion with BLiE schedule converges to Gaussian distribution more rapidly than that with default schedule. This finding indicates that BLiE finds a more efficient noise schedule with constraints on diffusion steps. Our results may provide helpful insights on the design of forward diffusions to human experts.

\section{Conclusion}
In this paper, we focus on continuous HPO problem. We formulate this problem as pure-exploration Lipschitz bandits, and propose BLiE as a solution. BLiE has three advantages: 1. BLiE has theoretical guarantees based only on a continuous assumption; 2. BLiE takes advantage of the continuity of the objective function to guide sampling; 3. BLiE can naturally work in parallel. Our empirical results demonstrate the superior performance of BLiE. We also apply BLiE to search noise schedule for diffusion model. Compared with standard linear schedule, BLiE schedule greatly improves the sampling speed without using additional techniques.

\newpage

\bibliography{pure-exploration-ref}
\bibliographystyle{icml2023}


\appendix

\section{Additional Related Works}
\label{app:related}

\textbf{Lipschitz Bandits:} The Lipschitz bandit problem was introduced as ``continuum-armed bandits'' \citep{agrawal1995continuum}, where the arm space is a compact interval. For this problem, \cite{kleinberg2005nearly} proved an $\Omega (T^{2/3})$ lower bound and introduced an algorithm that matches this lower bound. Under extra conditions on top of Lipschitzness, regret rate of $\widetilde{\mathcal{O}} (T^{1/2})$ was achieved \citep{auer2007improved, cope2009regret}. For compact doubling metric spaces, the Zooming bandit algorithm \citep{kleinberg2008multi}, the Hierarchical Optimistic Optimization (HOO) algorithm \citep{bubeck2011x}, and the BLiN algorithm \citep{fenglipschitz} were developed. Additionally, some attention has been focused on Lipschitz bandit problems where extra conditions are imposed. To name a few, \cite{bubeck2011lipschitz} study Lipschitz bandits for differentiable rewards, which enables algorithms to run without explicitly knowing the Lipschitz constants. \cite{10.1145/3412815.3416885} studied discretization-based Lipschitz bandit algorithms from a Gaussian process perspective. \cite{magureanu2014lipschitz} derive a new concentration inequality and study discrete Lipschitz bandits. The idea of robust mean estimators \citep{bickel1965some, alon1999space, bubeck2013bandits} was applied to the Lipschitz bandit problem to cope with heavy-tail rewards, leading to the development of a near-optimal algorithm for Lipschitz bandit with heavy-tailed rewards \citep{lu2019optimal}. Lipschitz bandits where a clustering is used to infer the underlying metric, has been studied by \cite{christina2019nonparametric}. Contextual Lipschitz bandits have also been studied \citep{slivkins2011contextual,krishnamurthy2019contextual}.

\textbf{Batched Bandits:} Urged by the recent prevalence of distributed computing, online learning problems with batched feedback has captured increasing attention \citep[e.g.,][]{cesa2013online}. In their seminal work, \cite{perchet2016batched} considered batched bandits with two arms, and provided a lower bound for the static grid. 
It was then generalized by \cite{gao2019batched} to finite-armed bandit problems. Soon afterwards, \cite{zhang2020inference} studied inference problems for batched bandits; \cite{esfandiari2021regret} studied batched linear bandits and batched adversarial bandits; \cite{han2020sequential} and \cite{ruan2021linear} provided solutions for batched contextual linear bandits; \cite{li2022gaussian} studied batched bandits from a Bayesian perspective. Batched dueling bandits \citep{agarwal2022batched} and batched Lipschitz bandits \citep{fenglipschitz} have also been studied. Parallel to the regret control regime, best arm identification with limited number of batches was also investigated \citep{agarwal2017learning,jun2016top}. Top-$k$ arm identification in the collaborative learning framework is also closely related to the batched setting, where the goal is to minimize the number of iterations (or communication steps) between agents. In this setting, tight bounds have been obtained recently \citep{tao2019collaborative,karpov2020collaborative}.

\section{Proof of Theorem \ref{thm:blin-upper}}\label{app:thm1}

\textbf{Theorem \ref{thm:blin-upper}.} If Assumption \ref{ass:hoeffding} and \ref{ass:lip} are satisfied, then output arm $\widetilde{x}^*$ of BLiE algorithm with total budget $T$, edge-length sequence $r_m=2^{-m}$, $\alpha=2L+2$ and $\beta$ satisfies
\begin{equation*}
    \Delta_{\widetilde{x}^*}\leq c\cdot T^{-\frac{1}{d_z+\beta}},
\end{equation*}
where $d_z$ is the zooming dimension and $c$ is a constant. In addition, BLiE needs no more than $\frac{1}{d_z+\beta}\log T$ batches to achieve this simple regret.

\begin{proof}
In the following analysis, we let $B$ be the total number of batches of the BLiE run ($B-1$ batches in the for-loop and $1$ clean-up batch). In each batch $m$, each cube is assigned with budget $n_m=\frac{1}{r_m^\beta}$. Lemma \ref{lem:eli} implies that each cube in $\mathcal{A}_m$ is a subset of $S((8L+8)r_m)$, so from the definition of zooming number and zooming dimension, we have $|\mathcal{A}_m|\leq N_{r_m}\leq C_z\cdot r_m^{-d_z}$. Therefore, the total budget of batch $m$ is upper bounded by $C_z\cdot r_m^{-(d_z+\beta)}$. Line 7 of Algorithm \ref{BLiN} yields that the sum of budgets of full $B$ batches is greater than $T$, so we have 
\begin{align*} 
    T\leq\sum_{m=1}^BC_z\cdot r_m^{-(d_z+\beta)}&\leq C_z\cdot\sum_{m=1}^B\left(2^{B-m}\cdot r_B\right)^{-(d_z+\beta)}\\
    &\leq\frac{C_z}{1-2^{-(d_z+\beta)}}r_B^{-(d_z+\beta)},
\end{align*} 
where the second inequality follows from $r_m=2^{-m}$, and thus $r_B\leq c_0\cdot T^{-\frac{1}{d_z+\beta}}$, where $c_0=\left(\frac{1-2^{-(d_z+\beta)}}{C_z}\right)^{-\frac{1}{d_z+\beta}}$. Lemma \ref{lem:eli} shows that $\Delta_x\leq(8L+8)r_B$ for any $x\in\cup_{C\in\mathcal{A}_B}C$, so we conclude that $\Delta_{\widetilde{x}^*}\leq c\cdot T^{-\frac{1}{d_z+\beta}}$, where $c=(8L+8)c_0$. Besides, since $r_B=2^{-B}$, no more than $B^*=\frac{1}{d_z+\beta}\log T-\log c_0$ batches are needed to achieve $r_B<c_0\cdot T^{-\frac{1}{d_z+\beta}}$ and $\Delta_{\widetilde{x}^*}\leq c\cdot T^{-\frac{1}{d_z+\beta}}$.
\end{proof}

\section{Proof of Lemma \ref{lem:eli}}
\label{app:lem1}

\textbf{Lemma \ref{lem:eli}.} For any $m\geq1$, any $C\in\mathcal{A}_m$ and any $x\in C$, simple regret of arm $x$ satisfies 
\begin{align*}
    \Delta_x\leq(4L+4)r_{m-1}.
\end{align*}

The proof of Lemma \ref{lem:eli} is based on the following two results.

\begin{lemma}\label{lem:concen}
    For any $m\geq1$, any cube $C\in\mathcal{A}_m$, selected arm $x_C$ and any arm $x\in C$,
    \begin{equation*}
        |\ell(x_C,n_m)-\mu(x)|\leq L\cdot r_m + n_m^{-\frac{1}{\beta}}.
    \end{equation*}
\end{lemma}
\begin{proof}
    Fix a cube $C\in\mathcal{A}_m$ and the selected arm $x_C$. Assumption \ref{ass:hoeffding} gives that
    \begin{align*}
        |\ell(x_C, n_m)-\mu(x_C)|\leq n_m^{-\frac{1}{\beta}}.
    \end{align*}
    By Lipschitzness of $\mu$, it is obvious that 
    \begin{align*}
        \left|\mu(x_C)-\mu(x)\right|\leq L\cdot r_m, \quad \forall x \in C.
    \end{align*} 
    Consequently, we have $|\ell(x_C,n_m)-\mu(x)|\leq L\cdot r_m + n_m^{-\frac{1}{\beta}}$.
\end{proof} 

\begin{lemma}\label{nel}
    The optimal arm $x^*=\arg\min\mu(x)$ is not eliminated in a BLiE run.
\end{lemma}

\begin{proof}
    We use $C^*_m$ to denote the cube containing $x^*$ in $\A_m$. Here we proof that $C^*_m$ is not eliminated in round $m$.
    	
    For any cube $C \in \A_m$ and $x\in C$, we have 
    \begin{align*} 
    	\ell(x_{C^*_m}, n_m) - \ell(x_C,n_m)\leq\mu(x^*) + n_m^{-\frac{1}{\beta}} + L\cdot r_m-\mu(x) +  n_m^{-\frac{1}{\beta}} + L\cdot r_m\leq(2L+2)r_m . 
    \end{align*} 
    Then from the elimination rule, $C^*_m$ is not eliminated. 
\end{proof}

\begin{proof}[Proof of Lemma \ref{lem:eli}]
    For $m=1$, the conclusion holds directly from the Lipschitzness of $\mu$. For $m>1$, let $C_{m-1}^*$ be the cube in $\A_{m-1}$ such that $x^* \in C_{m-1}^*$. From Lemma \ref{nel}, this cube $C_{m-1}^*$ is well-defined. 
    For any cube $C\in\mathcal{A}_m$ and $x\in C$, it is obvious that $x$ is also in the parent of $C$ (the cube in the previous round that contains $C$), which is denoted by $C_{par}$. 
    Thus for any $x \in C$, it holds that
    \begin{align*}
    	\Delta_x=\mu(x)-\mu^*\leq\ell(x_{C_{par}}, n_{m-1})+n_{m-1}^{-\frac{1}{\beta}}+L \cdot r_{m-1}-\ell(x_{C^*_{m-1}}, n_{m-1})+n_{m-1}^{-\frac{1}{\beta}}+L \cdot r_{m-1},
    \end{align*}
	where the inequality uses Lemma \ref{lem:concen}. 
    	
    Equality $n_{m-1}=\frac{1}{r_{m-1}^\beta}$ gives that 
    \begin{align*}
        \Delta_x 
        &\le\ell(x_{C_{par}}, n_{m-1})-\ell(x_{C^*_{m-1}}, n_{m-1})+(2L+2)r_{m-1}. 
    \end{align*}
    It is obvious that $\ell(x_{C^*_{m-1}}, n_{m-1})\geq\ell_{m-1}^{\min}$. Moreover, since the cube $C_{par}$ is not eliminated, from the elimination rule we have
    \begin{align*}
        \ell(x_{C_{par}}, n_{m-1})-\ell_{m-1}^{\min}\le (2L+2)r_{m-1}.
    \end{align*}
    Hence, we conclude that $\Delta_x\leq(4L+4)r_{m-1}$.
\end{proof}

\section{Definition \ref{def:ace} and Theorem \ref{thm:ace} without Decreasing Assumption}\label{app:race}

\textbf{Defnintion \ref{def:ace}.} For a problem with ambient dimension $d$, zooming dimension $d_z$ and time horizon $T$, we denote $c_1=\frac{d_z+\beta-1}{(d_z+\beta)(d+\beta)}\log T$ and $c_{i+1} = {\eta c_i }$ for any $i\ge 1$, where $\eta=\frac{d+1-d_z}{d+\beta}$. Then we let $\alpha_n=\lfloor\sum_{i=1}^nc_i\rfloor$, $\beta_n=\lceil\sum_{i=1}^nc_i\rceil$, and inductively define ACE Sequence $\{r_m\}_{m\in\mathbb{N}}$ as $r_m=\min\{r_{m-1}, 2^{-\alpha_k}\}$ for $m=2k-1$ and $r_m=2^{-\beta_k}$ for $m=2k$. Since every $c_i$ is positive, it is easy to see $\{r_m\}$ is a decreasing sequence. If there exists $m$ such that $r_m=r_{m-1}$, then we skip the $m$-th batch when using ACE Sequence in BLiE.

\textbf{Theorem \ref{thm:ace}.} If Assumption \ref{ass:hoeffding} and \ref{ass:lip} are satisfied, then output arm $\widetilde{x}^*$ of BLiE algorithm with total budget $T$, ACE Sequence $\{r_m\}$, $\alpha=2L+2$ and $\beta$ satisfies $$\Delta_{\widetilde{x}^*}\leq c\cdot T^{-\frac{1}{d_z+\beta}},$$
where $d_z$ is the zooming dimension and $c$ is a constant. In addition, BLiE needs $\mathcal{O}(\log\log T)$ batches to achieve this simple regret.

\begin{proof}
We let $B$ be the total number of batches of the BLiE run ($B-1$ batches in the for-loop and $1$ clean-up batch), and $N_m$ be the total budget of batch $m$. In the following analysis, we bound $N_m$ for $m=2k-1$ and $m=2k$ separately, and then obtain the simple regret upper bound. For $m$ such that $r_m=r_{m-1}$, we skip batch $m$ and define $N_m=0$. Then the following bounds are still hold. Thus, without loss of generality, we assume $r_m<r_{m-1}$.

Firstly, we consider the case $m=2k-1$. Recall that $r_m=\min\{r_{m-1}, 2^{-\alpha_k}\}$, so $r_m<r_{m-1}$ yields that $r_m=2^{-\alpha_k}$. For convenience, we let $\widetilde{r}_k=2^{-\sum_{i=1}^kc_i}$, and thus we have $\widetilde{r}_{k-1}\geq r_{m-1}\geq r_m\geq \widetilde{r}_{k}$. Recall that $\mathcal{A}_{m-1}^+$ is set of cubes not eliminated in batch $m-1$, and $\cup_{C\in\mathcal{A}_{m-1}^+}C=\cup_{C^\prime\in\mathcal{A}_m}C^\prime$. Lemma \ref{lem:eli} implies that each cube in $\mathcal{A}_{m-1}^+$ is a subset of $S((4L+4)r_{m-1})$, and thus $\left|\mathcal{A}_{m-1}^+\right|\leq N_{r_{m-1}}\leq C_z\cdot r_{m-1}^{-d_z}$. The total budget of batch $m$ is
\begin{align*}
    N_m=|\mathcal{A}_m|\cdot n_m = \left(\frac{r_{m-1}}{r_m}\right)^d\left|\mathcal{A}_{m-1}^+\right|\cdot n_m
    \leq C_z\cdot\frac{r_{m-1}^{d+1-d_z}\cdot r_m^{-d-\beta}}{r_{m-1}} 
    \leq C_z\cdot\frac{\widetilde{r}_{k-1}^{d+1-d_z}\cdot\widetilde{r}_k^{-d-\beta}}{r_{m-1}}.
\end{align*}
For the numerator, we have
\begin{align*}
    \widetilde{r}_{k-1}^{d+1-d_z}\cdot\widetilde{r}_k^{-d-\beta}=2^{-\left(\sum_{i=1}^{k-1}c_i\right)(d+1-d_z)+\left(\sum_{i=1}^{k}c_i\right)(d+\beta)}=2^{\left(\sum_{i=1}^{k-1}c_i\right)(d_z+\beta-1)+c_m(d+\beta)}.
\end{align*}
Define $C_m=\left(\sum_{i=1}^{m-1}c_i\right)(d_z+\beta-1)+c_m(d+\beta)$. Since $c_m=c_{m-1}\cdot\frac{d+1-d_z}{d+\beta}$, calculation shows that $C_m = (\sum_{i=1}^{m-2} c_i) (d_z+\beta-1)  + c_{m-1}(d+\beta)+ c_{m-1}(d_z+\beta-1-d-\beta)+c_m(d+\beta)=C_{m-1}$. Thus for any $m$, we have $C_M=C_1=\frac{d_z+\beta-1}{d_z+\beta}\log T$. Hence,
\begin{equation*}
    N_m\leq C_z\cdot2^{\frac{d_z+\beta-1}{d_z+\beta}\log T}/r_{m-1}=C_z\cdot T^{\frac{d_z+\beta-1}{d_z+\beta}}/r_{m-1}.
\end{equation*}

Secondly, we consider the case $m=2k$. Recall that $r_{m-1}=\min\{r_{m-2}, 2^{-\alpha_k}\}$. It is easy to verify that $r_{m-2}<2^{-\alpha_k}$ only happens when $r_{m-2}=r_{m}$, so $r_m<r_{m-1}$ yields that $r_{m-1}=2^{-\alpha_k}$. Lemma \ref{lem:eli} implies that each cube in $\mathcal{A}_m$ is a subset of $S((8L+8)r_m)$. Similar argument to the first case shows that $\left|\mathcal{A}_m\right|\leq C_z\cdot r_m^{-d_z}$ and $N_m\leq C_z\cdot r_m^{-(d_z+\beta)}$.

Line 7 of Algorithm \ref{BLiN} ensures that the sum of budgets of full $B$ batches is greater than $T$. Therefore, combining the above two cases, we have
\begin{align}\label{app-ace1}
    T\leq&\sum_{\substack{m=2k-1, \\ m\leq B }}C_z \cdot\frac{T^{\frac{d_z+\beta-1}{d_z+\beta}}}{r_{m-1}}+\sum_{\substack{m=2k,\\m\leq B}}C_z\cdot \frac{1}{r_m^{d_z+\beta}}.
\end{align}

The rounding step in Definition \ref{def:ace} yields that $r_m\leq\frac{r_{m-1}}{2}$ for any $m$. If $B$ is odd, from (\ref{app-ace1}) we have 
\begin{equation}\label{app-ace2}
    T\leq2C_z\cdot T^{\frac{d_z+\beta-1}{d_z+\beta}}\cdot\frac{1}{r_{B-1}}+2C_z\cdot\frac{1}{r_{B-1}^{d_z+\beta}}.
\end{equation}

We set $c_r=\max\{(4C_z)^\frac{1}{d_z+\beta}, 4C_z\}$, then $\frac{1}{c_r^{d_z+\beta}}+\frac{1}{c_r}\leq\frac{1}{2C_z}$. If $r_{B-1}>c_r\cdot T^{-\frac{1}{d_z+\beta}}$, then calculation shows that 
\begin{align*}
    2C_z\cdot T^{\frac{d_z+\beta-1}{d_z+\beta}}\cdot\frac{1}{r_{B-1}}+2C_z\cdot\frac{1}{r_{B-1}^{d_z+\beta}}<T,
\end{align*}
which contradicts (\ref{app-ace2}). Therefore, we have $r_{B-1}\leq c_r\cdot T^{-\frac{1}{d_z+\beta}}$. Lemma \ref{lem:eli} shows that $\Delta_x\leq(4L+4)r_{B-1}$ for any $x\in\cup_{C\in\mathcal{A}_B}C$, so we have $\Delta_{\widetilde{x}^*}\leq c\cdot T^{-\frac{1}{d_z+\beta}}$, where $c=(8L+8)\cdot\max\{c_r, 1\}$. 

If $B$ is even, from (\ref{app-ace1}) we have
\begin{equation*}
    T\leq2C_z\cdot T^{\frac{d_z+\beta-1}{d_z+\beta}}\cdot\frac{1}{r_{B}}+2C_z\cdot\frac{1}{r_{B}^{d_z+\beta}}.
\end{equation*}
Similar arguments yield that $r_B\leq c_r\cdot T^{-\frac{1}{d_z+\beta}}$. Lemma \ref{lem:eli} shows that $\Delta_x\leq(4L+4)r_{B-1}$ for any $x\in\cup_{C\in\mathcal{A}_B}C$. Moreover, since $B$ is even, from definition of ACE Sequence we have $r_{B-1}\leq2r_B$. Consequently, in this case we also have $\Delta_{\widetilde{x}^*}\leq c\cdot T^{-\frac{1}{d_z+\beta}}$.

Finally, we consider the communication bound. For any $B^*$, $\widetilde{r}_{B^*}=2^{-\sum_{i=1}^{B^*}c_i}=2^{-c_1\cdot\frac{1-\eta^{B^*}}{1-\eta}}=T^{-\frac{1}{d_z+\beta}}\cdot T^{\frac{\eta^{B^*}}{d_z+\beta}}$. Then by choosing $B^*\geq\frac{\log\log T-\log(d_z+\beta)}{\log\frac{d+\beta}{d+1-d_z}}$, we have $\frac{\eta^{B^*}}{d_z+\beta}\leq\log T$ and $\widetilde{r}_{B^*}\leq T^{-\frac{1}{d_z+\beta}}$. Definition \ref{def:ace} shows that $r_m<\widetilde{r}_{B^*}$ for any $m\geq2B^*$. As a consequence, no more than $\hat{B}=\frac{2\log\log T}{\log\frac{d+\beta}{d+1-d_z}} + 1$ batches are needed to achieve $r_{\hat{B}-1}\leq T^{-\frac{1}{d_z+\beta}}$
and 
\begin{equation*}
    \Delta_{\widetilde{x}^*}\leq(4L+4)r_{\hat{B}-1}\leq(4L+4)T^{-\frac{1}{d_z+\beta}}\leq c\cdot T^{-\frac{1}{d_z+\beta}},
\end{equation*}
where the first inequality follows from Lemma \ref{lem:eli}.
\end{proof}

\section{Proof of Theorem \ref{thm:lower-us}}\label{app:us}
The pseudo code of uniform search is presented below.

\begin{algorithm}[H]
	\caption{Uniform Search} 
	\label{uni-exploration} 
	\begin{algorithmic}[1]  
		\STATE \textbf{Input.} Arm set $\mathcal{X}=[0,1]^d$; Total budget $T$; Grid length $r$; $n = Tr^d$.
		\STATE Equally partition $\mathcal{X}$ to $N=\frac{1}{r^d}$ subcubes and define $\mathcal{A}$ as the collection of these subcubes. 
		\FOR{$C\in\A$}
		    \STATE Uniformly choose an arm $x_C\in C$. 
		    \STATE Evaluate arm $x_C$ with budget $n$. Recive the loss $\ell(x_C, n)$.
		\ENDFOR
		\STATE Compute $C^*=\mathop{\arg\min}_{C\in\A}\ell(x_C, n)$.\label{ue:opt-cube}
		\STATE \textbf{Output} $\widetilde{x}^*=x_{C^*}$.
	\end{algorithmic} 
\end{algorithm}

\textbf{Theorem \ref{thm:lower-us}.} For any total budget $T$, dimension $d$ and grid length $r$, there exists an instance with zooming dimension $d_z=0$ such that the uniform search strategy returns an arm $\widetilde{x}^*$ with optimal gap $\E\Delta_{\widetilde{x}^*}\geq \frac{1}{2}T^{-\frac{1}{d+\beta}}$.

\begin{proof}
We construct a problem instance such that $\mathcal{X}=[0,1]^d$ and $\mu(x)=f(\|x\|_\infty)\triangleq C+\|x\|_\infty$ for some constant $C>0$. For this instance, we have $x^*=0$ and $\Delta_x=\mu(x)-\mu(x^*)=\|x\|_\infty$. This instance satisfies Assumption \ref{ass:lip} with $L=1$. Similar arguments to the example in Section \ref{sec:zooming} yields that the zooming dimension of this instance equals to $0$, with zooming constant $C_z=16^d$.

After given edge length $r$, $\mathcal{X}$ is equally partitioned into $N=\frac{1}{r^d}$ cubes $C_1\cdots C_N$. Additionally, we define the marginal grid point $g_i=\frac{i}{N}$. In the following analysis, we consider the situations where $r\geq T^{-\frac{1}{d+\beta}}$ and $r<T^{-\frac{1}{d+\beta}}$ separately.

If $r\geq T^{-\frac{1}{d+\beta}}$, we define $\ell(x,n)=\mu(x)$ for any $x$ and $n$. For this instance, it is easy to see that the optimal cube $C^*$ in Line \ref{ue:opt-cube} is $C_1=[0,g_1]^d$. Therefore, the output arm is uniformly selected from $C_1$, and
\begin{equation}\label{ue:lower-bound1}
    \E \Delta_{\widetilde{x}^*}=\E_{x\sim \mathrm{Unif}(C_1)}\Delta_x>\frac{r}{2}\geq\frac{1}{2}T^{-\frac{1}{d+\beta}}.
\end{equation}
If $r<T^{-\frac{1}{d+\beta}}$, then each cube is played for $n=Tr^d<T^{\frac{\beta}{d+\beta}}$ times. Since the edge length $g_i-g_{i-1}$ equals to $r$ for each $i$, there exists an integer $0<k_0\leq N$ such that $\frac{1}{2}T^{-\frac{1}{d+\beta}}\leq f(g_{k_0})-\mu(x^*)\leq T^{-\frac{1}{d+\beta}}$. We set $\mathcal{A}_1=\{C_i:\;1\leq i\leq N,\;C_i\subset[0, g_{k_0}]^d\} $ and $\mathcal{A}_2=\{C_i:\;1\leq i\leq N,\;C_i\not\subset[0, g_{k_0}]^d\} $ and then define $\ell(x,n)=\mu(x)+n^{-\frac{1}{\beta}}$ for $x\in\cup_{C\in\mathcal{A}_1}C$, and $\ell(x,n)=\mu(x)-n^{-\frac{1}{\beta}}$ for $x\in\cup_{C\in\mathcal{A}_2}C$.

Since $n<T^{\frac{\beta}{d+\beta}}$, for any $x\in\cup_{C\in\mathcal{A}_1}C$, we have
\begin{equation}\label{error:first-grid}
    \ell(x,n)=\mu(x)+n^{-\frac{1}{\beta}}>\mu(x^*)+\left(T^{\frac{\beta}{d+\beta}}\right)^{-\frac{1}{\beta}}=C+T^{-\frac{1}{d+\beta}}.
\end{equation}
Since $f(g_{k_0})-\mu(x^*)\leq T^{-\frac{1}{d+\beta}}$ and $r<T^{-\frac{1}{d+\beta}}$, we have $f(g_{k_0+1})=f(g_{k_0})+r\leq \mu(x^*)+2T^{-\frac{1}{d+\beta}}$. Therefore, for any $x\in [g_{k_0},g_{k_0+1}]^d$,
\begin{align}\label{err:second-grid}
        \ell(x,n)&=\mu(x)-n^{-\frac{1}{\beta}}<f(g_{k_0+1})-\left(T^{\frac{\beta}{d+\beta}}\right)^{-\frac{1}{\beta}}\leq C+2T^{-\frac{1}{d+\beta}}-T^{-\frac{1}{d+\beta}}=C+T^{-\frac{1}{d+\beta}}.
\end{align}
Combining (\ref{error:first-grid}) and (\ref{err:second-grid}), we show that the loss $\ell(x,n)$ for $x\in\cup_{C\in\mathcal{A}_1}C$ is sub-optimal, and the optimal cube $C^*$ in Line 7 belongs to $\mathcal{A}_2$. From the definition of the instance, for any $x\in\cup_{C\in\mathcal{A}_2}C$, the optimal gap $\Delta_x\geq f(g_{k_0})-\mu(x^*)\geq\frac{1}{2}T^{-\frac{1}{d+\beta}}$. As a consequence, we have
\begin{equation}\label{ue:lower-bound2}
    \E\Delta_{\widetilde{x}^*}\geq\frac{1}{2}T^{-\frac{1}{d+\beta}}.
\end{equation}

Finally, combining (\ref{ue:lower-bound1}) and (\ref{ue:lower-bound2}), we arrive at the conclusion of the theorem.
\end{proof}

\section{Proof of Theorem \ref{thm:lower-hb}}\label{app:lower-hb}

The pseudo code of random-search strategy is presented below.

\begin{algorithm}[H]
    \caption{Random-Search-Based Algorithm} 
    \label{random-search} 
    \begin{algorithmic}[1]  
        \STATE \textbf{Input.} Arm set $\mathcal{X}=[0,1]^d$; Total budget $T$; Number of selected arms $N\leq T$.
        \STATE Select $N$ arms $\mathcal{X}_s=\{x_i\}_{i=1}^N\subseteq\mathcal{X}$, where each $x_i$ is uniformly sampled from $\mathcal{X}$.
        \STATE Choose an arm $\widetilde{x}^*\in\mathcal{X}_s$ according to some policy.
        \STATE \textbf{Output} $\widetilde{x}^*$.
    \end{algorithmic} 
\end{algorithm}

\textbf{Theorem \ref{thm:lower-hb}.} 
For any total budget $T$, dimension $d$, zooming dimension $d_z$ and number of selected arms $N$, there exists an instance such that any random-search-based algorithm returns an arm $\widetilde{x}^*$ with optimal gap $\E\Delta_{\widetilde{x}^*}\geq c\cdot T^{-\frac{1}{d-d_z}}$, where $c$ is a constant.

\begin{proof}
We consider an instance with ambient dimension $d$ and zooming dimension $d_z$. Let $x_r$ be a uniformly chosen arm and $\mu_r=\mu(x_r)$. Then we have $\Pr(\mu_r<\mu^*+\varepsilon)=\m(\{x:\mu(x)-\mu^*<\varepsilon\})=\m(S(\varepsilon))$. Definition of zooming number and zooming dimension yields that $S(\varepsilon)$ contains $N_{\frac{\varepsilon}{8L+8}}$ standard cubes with edge length $\frac{\varepsilon}{8L+8}$, and $N_{\frac{\varepsilon}{8L+8}}\leq C_z\cdot\left(\frac{\varepsilon}{8L+8}\right)^{-d_z}$. We denote the set of these standard cubes as $\mathcal{A}=\{C_1,\cdots, C_{N_{\frac{\varepsilon}{8L+8}}}\}$. 

For any standard cube $C$ with edge length $\frac{\varepsilon}{8L+8}$ such that $C\notin \mathcal{A}$, there exists some $x_C\in C$ such that $\Delta_{{x_C}}\geq\varepsilon$. Since $\mu$ is $L$-Lipschitz, for any $x\in C$, we have
\begin{equation*}
    \Delta_x\geq\Delta_{x_C}-L\cdot\frac{\varepsilon}{8L+8}\geq\frac{7}{8}\varepsilon.
\end{equation*}
As a consequence, $S(\frac{7}{8}\varepsilon)$ is covered by $\mathcal{A}$, and the measure is bounded by $\m(S(\frac{7}{8}\varepsilon))\leq C_z\cdot\left(\frac{\varepsilon}{8L+8}\right)^{d-d_z}$. Therefore, we have

\begin{equation}\label{opt-region-bound}
    \Pr\left(\mu_r<\mu^*-\frac{7}{8}\varepsilon\right)\leq C_z\cdot\left(\frac{\varepsilon}{8L+8}\right)^{d-d_z}.
\end{equation}
The following analysis is similar to the lower bound proof in \cite{carpentier2015simple}. We set $\varepsilon=c_0\cdot T^{-\frac{1}{d-d_z}}$, where $c_0=\frac{8L+8}{C_z^\frac{1}{d-d_z}}$, and (\ref{opt-region-bound}) gives that $\Pr\left(\mu_r<\mu^*-\frac{7c_0}{8}\cdot T^{-\frac{1}{d-d_z}}\right)\leq\frac{1}{T}$. Then for the $N$ different arms selected in Algorithm \ref{random-search}, we have
\begin{align*}
    \Pr\left(\mu(x_i)\geq\mu^*+\frac{7c_0}{8}\cdot T^{-\frac{1}{d-d_z}},\;\forall\;1\leq i\leq N\right)\geq\left(1-\frac{1}{T}\right)^N\geq\left(1-\frac{1}{T}\right)^T\geq\frac{1}{4}.
\end{align*}
As a consequence, with probability more than $\frac{1}{4}$, all selected arms have optimal gap larger than $\frac{7c_0}{8}\cdot T^{-\frac{1}{d-d_z}}$, and therefore, with probability larger than $\frac{1}{4}$, the output optimal gap of random search is lower bounded by $\frac{7c_0}{8}\cdot T^{-\frac{1}{d-d_z}}$. Consequently, we have $\E\Delta_{\widetilde{x}^*}\geq\frac{7c_0}{32}\cdot T^{-\frac{1}{d-d_z}}$.
\end{proof}

\section{More Samples Generated Using Different Schedules and Diffusion Steps}
\begin{figure}[ht]
	\centering
	\subfigure[BLiE Result, $T=100$]{
		\includegraphics[width=6cm]{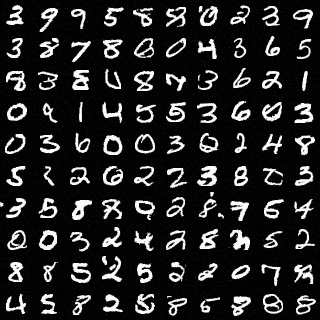}
	}\hspace{+10mm}
	\subfigure[Linear, $T=100$]{
		\includegraphics[width=6cm]{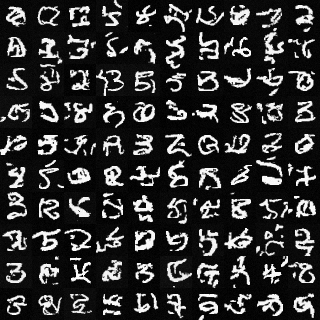}
	}\\
	\subfigure[BLiE Result, $T=200$]{
		\includegraphics[width=6cm]{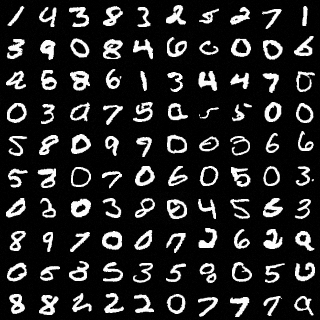}
	}\hspace{+10mm}
	\subfigure[Linear, $T=200$]{
		\includegraphics[width=6cm]{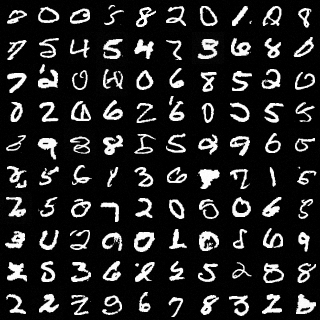}
	}\\
	\subfigure[BLiE Result, $T=1000$]{
		\includegraphics[width=6cm]{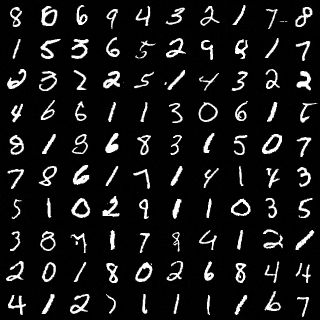}
	}\hspace{+10mm}
	\subfigure[Linear, $T=1000$]{
		\includegraphics[width=6cm]{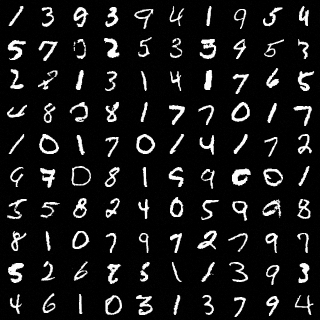}
	}
    \caption{MNIST samples generated using different noise schedules and diffusion steps.}
\end{figure}

\end{document}